%% file: icml_2020_workshop_chi.tex
\documentclass{article}

\usepackage{microtype}
\usepackage{graphicx}
\usepackage{subcaption}
\usepackage{booktabs} 
\usepackage{amsthm}
\usepackage{amsmath}
\usepackage{amsfonts}
\usepackage{mathtools}
\usepackage{bbm}
\usepackage{appendix}
\usepackage{nicefrac}

\newtheorem{theorem}{Theorem}[section]
\newtheorem{lemma}[theorem]{Lemma}
\newtheorem{problem}{Problem}[section]

\DeclareMathOperator*{\argmax}{arg\,max}
\DeclareMathOperator*{\argmin}{arg\,min}

\DeclareMathOperator*{\denv}{\mathcal{D}_{\text{env}}}
\DeclareMathOperator*{\dmodel}{\mathcal{D}_{\text{model}}}

\DeclareMathOperator*{\piphi}{\pi_{\phi}}

\graphicspath{{./figure/}}

\usepackage{hyperref}

\usepackage[accepted]{icml2020}

\icmltitlerunning{Maximum Entropy Model Rollouts}

\begin{document}

\twocolumn[
\icmltitle{Maximum Entropy Model Rollouts: \\Fast Model Based Policy Optimization without Compounding Errors}





\begin{icmlauthorlist}
\icmlauthor{Chi Zhang}{usc_cs}
\icmlauthor{Sanmukh Rao Kuppannagari}{usc_ee}
\icmlauthor{Viktor K Prasanna}{usc_ee}
\end{icmlauthorlist}

\icmlaffiliation{usc_cs}{Department of Computer Science, University of Southern California, Los Angeles, CA}
\icmlaffiliation{usc_ee}{Department of Electrical and Computer Engineering, University of Southern California, Los Angeles, CA}

\icmlcorrespondingauthor{Chi Zhang}{zhan527@usc.edu}

\icmlkeywords{Model-based Reinforcement Learning, Compounding Errors, Sample Efficiency, Computation Efficiency, Maximum Entropy}

\vskip 0.3in
]



\printAffiliationsAndNotice{}  

\input{content/content.tex}

\section*{Acknowledgements}
This work has been sponsored by the U.S. Army Research Office (ARO) under award number W911NF1910362 and the U.S. National Science Foundation (NSF) under award number 1911229.

\bibliography{bib/icml_2020_workshop_chi}
\bibliographystyle{icml2020}

\appendix

\input{content/appendix.tex}

%
%
%

\end{document}

%% file: content/content.tex
\begin{abstract}
  Model usage is the central challenge of model-based reinforcement learning. Although dynamics model based on deep neural networks provide good generalization for single step prediction, such ability is over exploited when it is used to predict long horizon trajectories due to compounding errors. In this work, we propose a Dyna-style model-based reinforcement learning algorithm, which we called Maximum Entropy Model Rollouts (MEMR). To eliminate the compounding errors, we only use our model to generate single-step rollouts. Furthermore, we propose to generate \emph{diverse} model rollouts by non-uniform sampling of the environment states such that the entropy of the model rollouts is maximized. We mathematically derived the maximum entropy sampling criteria for one data case under Gaussian prior. To accomplish this criteria, we propose to utilize a prioritized experience replay. Our preliminary experiments in challenging locomotion benchmarks show that our approach achieves the same sample efficiency of the best model-based algorithms, matches the asymptotic performance of the best model-free algorithms, and significantly reduces the computation requirements of other model-based methods.
\end{abstract}

\section{Introduction}
Model-based reinforcement learning (MBRL) \cite{mbpo, steve, slbo, pets} shows competitive performance compared with best model-free reinforcement learning (MFRL) algorithms \cite{ppo, trpo, dqn, sac, sac_algo_apps} with significantly fewer environment samples on challenging robotics locomotion benchmarks \cite{mujoco}.  A MFRL algorithm learns complex skills by maximizing a scalar reward designed by human engineering. However, to obtain promising performance a large number of environment interactions are needed which may take a long time in real-world applications. In such cases, MBRL is appealing due to its superior sample efficiency that relies on the generalization of a learned predictive dynamics model. However, the quality of the policy trained on imagined trajectories is often worse asymptotically than the best MFRL counterparts due to the imperfect models.

Recently, \cite{mbpo} proposed Model-based Policy Optimization (MBPO), including a theoretical framework that encourages short-horizon model usage based on an optimistic assumption of a bounded model generalization error given policy shift. Although empirical studies have shown support evidence, this property is hard to guarantee in the whole state distribution. Moreover, \emph{uniform} sampling of the environment states to generate branched model rollouts degrades the \emph{diversity} of the model dataset, especially when the policy shift is small, which makes the policy updates inefficient.

Our main contribution is a practical algorithm, which we called Maximum Entropy Model Rollouts (MEMR) based on the aforementioned insights. The differences between MEMR and MBPO are: 1) MEMR follows Dyna \cite{dyna} that only generates single-step model rollouts while MBPO encourages generating short-horizon model rollouts.
The generalization ability of MEMR is strictly guaranteed by supervised machine learning theory, which can be empirically estimated by validation errors \cite{understanding_ml}. 2) MEMR utilizes a prioritized experience replay \cite{prioritized_experience_replay} to generate \emph{max-diversity} model rollouts for efficient policy updates.
We validate this idea on challenging locomotion benchmarks \cite{mujoco} and the experimental results show that MEMR matches the asymptotic performance and sample efficiency of MBPO \cite{mbpo} while significantly reducing the number of policy updates and model rollouts leading to faster learning speed.

\section{Preliminaries}
 Reinforcement Learning algorithms aim to solve Markov Decision Process (MDP) with unknown dynamics. A Markov decision process (MDP) \cite{rl_intro} is defined as a tuple $<\mathcal{S}, \mathcal{A}, R, P, \mu>$, where $\mathcal{S}$ is the set of states, $\mathcal{A}$ is the set of actions, $R(s, a, s'):\mathcal{S}\times\mathcal{A}\times\mathcal{S}\rightarrow \mathbb{R}$ defines the intermediate reward when the agent transits from state $s$ to $s'$ by taking action $a$, $P(s'|s,a):\mathcal{S}\times\mathcal{A}\times\mathcal{S}\rightarrow [0, 1]$ defines the probability when the agent transits from state $s$ to $s'$ by taking action $a$, $\mu: \mathcal{S}\rightarrow[0, 1]$ defines the starting state distribution. The objective of reinforcement learning is to select policy $\pi: \mathcal{\mu}\rightarrow P(A)$ such that 
\begin{equation}
J(\pi)= \underset{\substack{s_0\sim\mu,a_t\sim \pi(\cdot|s_t)\\s_{t+1}\sim P(\cdot|s_t,a_t)}}{\mathbb{E}}[\sum_{t=0}^{\infty}\gamma^t R(s_t,a_t,s_{t+1})]
\label{eq:mdp_objective}
\end{equation}
is maximized.

\subsection{Prioritized Experience Replay}
Prioritized experience replay \cite{prioritized_experience_replay} is introduced to increase the learning efficiency of DQN \cite{dqn}, where the probability of each transition is proportional to the absolute TD error \cite{q_learning}. To avoid overfitting, stochastic prioritization is utilized and the bias is corrected via annealed importance sampling. In this work, we adopt the same idea with a custom prioritization criteria such that the joint entropy of the state and action in the model dataset is maximized.

\subsection{Model-based Policy Optimization}

Model-based policy optimization (MBPO) \cite{mbpo} achieves state-of-the-art sample efficiency and matches the asymptotic performance of MFRL approaches. MBPO optimizes a policy with soft actor-critic (SAC) \cite{sac} under the data distribution collected by unrolling the learned dynamics model using the current policy. However, the sample efficiency comes at the cost of 2.5x to 5x increased number of policy updates compared with SAC \cite{sac} and a large number of model rollouts, that significantly decreases the training speed. To mitigate this bottleneck, we analyze the model usage and model rollout distribution and propose insights on how to improve MBPO to obtain better computation efficiency.

\paragraph{Model usage.} In MBPO, learned dynamics model is used to generate branched model rollouts with short horizons \cite{mbpo}. Although \cite{mbpo} presented theoretical analysis to bound the policy performance trained using model generate rollouts, the over exploitation of model generalization can't be eliminated. In this work, \emph{one of our core idea is that we only rely on learned model to generate one-step rollouts, in which case we interpret it as model-based exploration.} The nice property of this model usage is the natural bounded model generalization error, which can be estimated in practice by the validation dataset \cite{understanding_ml}.

\paragraph{Model rollout distribution.} Uniform sampling of true states \footnote{States encountered in real environment as opposed to imagined states that are generated by the model.} to generate model rollouts is adopted in MBPO \cite{mbpo}. This potentially generates large amount of similar data when the policy and the learned model changes slowly as training progresses. As result, the efficiency of the policy updates is deteriorated. In this work, \emph{we propose to sample true states to generate single-step model rollouts such that the joint entropy of the state and action of the model dataset is maximized.} The intuition is to increase the "diversity" of the model dataset, from which the policy can benefit for efficient learning.

\section{Maximum Entropy Model Rollouts}
In this section, we unveil the technical details of our Maximum Entropy Model Rollouts (MEMR) for model based policy optimization. First, we propose the  Maximum Entropy Sampling Theorem to help understand the choice of our prioritization criteria. Based on the theoretical analysis, we propose a practical implementation of this idea and discuss the challenges posed by runtime complexity along with their fixes.


\subsection{Maximum Entropy Sampling Criteria}
We begin by considering the following problem definition:
\begin{problem}[Maximum Entropy Sampling]
  Let $\denv=\{s_i\}_{i=1}^{N_{\text{env}}}$ be the collection of all the states in the environment dataset. Let $\dmodel=\{(s,a)\}_{j=1}^{N_{\text{model}}}$ be the collection of all the state-action pairs in the model dataset\footnote{The tasks considered in this work are deterministic so we omit $s'$ for simplicity.}. Assume for each state in $\denv$, we sample action $a_i\sim \pi_{\phi}(\cdot|s_i)$ using the current policy denoted as $\mathcal{D}_{\text{sample}}=\{(s_i, a_i)\}_{i=1}^{N_{\text{env}}}$. Assume we parameterize the policy distribution derived from the model dataset as a Gaussian distribution with diagonal covariance: $\pi_{\psi}(a_i|s_i)=\mathcal{N}(\mu_{\psi}(s_i), \Sigma_{\psi}(s_i))$. Let the joint entropy of the state-action in the model dataset be $H(S,A)$. Now we select $(s_k, a_k)$ from $D_{\text{sample}}$ and add it to the $\dmodel$. Let the joint entropy of the new $\mathcal{D}'_{\text{model}}=\dmodel\cup \{(s_k,a_k)\}$ be $H(S',A')$, the optimal sampling criteria problem is to choose index $k$ such that $H(S',A')-H(S,A)$ is maximized.
\label{prob:optimal_sampling}
\end{problem}

\begin{theorem}[Maximum Entropy Sampling Theorem]
 Assume $N_{\text{model}}\gg 1$ such that the state distribution of $\dmodel$ and $\mathcal{D}'_{\text{model}}$ are identical, then
 \begin{align}
 k=\argmin_i \log(\sqrt{2\pi}\pi_{\psi}(a_i|s_i) \sigma(\pi_{\psi}(\cdot|s_i)))
 \end{align}
 where $\pi_{\psi}(a_i|s_i)$ is the probability of model data policy at $(s_i,a_i)$, $\sigma(\pi_{\psi}(\cdot|s_i))$ is the standard deviation of the conditional distribution at $s_i$.
\label{the:optimal_sampling_main}
\end{theorem}
\begin{proof}
  See Appendix~\ref{sec:optimal_sampling_proof}, Theorem~\ref{the:optimal_sampling_appendix}.
\end{proof}

\subsection{Practical Implementation}
Theorem~\ref{the:optimal_sampling_main} provides a mathematically justified criteria to select states from the environment dataset for rollout generation to maximize the "diversity" of the model dataset, yet it poses several practical challenges to implement: 1) It requires a full sweep of all the states in the environment dataset before each sampling, which is $O(N_{env})$. This is problematic because $N_{env}$ grows linearly as training progresses. 2) Stochastic gradient descent assumes uniform sampling of the data distribution whereas prioritized sampling breaks this assumption and introduces bias. 3) Training the model data distribution to converge is expensive but crucial before evaluating the priority. A complete algorithm that handles the aforementioned practical challenges is presented in Algorithm~\ref{alg:memr}.

\begin{algorithm}[!t]
  \caption{Maximum Entropy Model Rollouts for Model-Based Policy Optimization}
  \label{alg:memr}
  \begin{algorithmic}[1]
    \STATE Initialize environment dataset $\denv$ and model dataset $\dmodel$
    \STATE Initialize SAC policy $\piphi$, predictive model $p_{\theta}$ and model derived policy distribution $\pi_{\psi}$
    \FOR{$t=1:\text{total\_num\_steps}$}
    \IF{$t\% \text{model\_update\_freq}==0$}
    \STATE Train model $p_{\theta}$ on $\denv$ via maximum likelihood
    \ENDIF
    \STATE Sample $a_t\sim\piphi(\cdot|s_t)$; Execute $a_t$ in the environment and observe $s_{t+1}$
    \STATE Compute priority $p_t$ according to Equation~\ref{eq:priority}; add $(s_t,a_t,s_{t+1},p_t)$ to $\denv$
    \FOR{$j=1:M$}
    \STATE Sample $s_j\sim P(j)=\nicefrac{p^{\alpha}_j}{\sum_{i}p^{\alpha}_j}$ from $\denv$
    \STATE Compute importance-sampling weight ${w_j=(N\cdot P(j))^{-\beta}/\max_i w_i}$
    \STATE Sample $a_j\sim\piphi(\cdot|s_j)$; Perform one-step rollout using $p_{\theta}$ and obtain $\hat{s}^{'}_j$.
    \ENDFOR
    \STATE Add $\{(s_j, a_j,\hat{s}^{'}_j,w_j)\}_{j=1}^{M}$ to the next segment in $\dmodel$
    \STATE Update $\pi_{\psi}$ on $\{(s_j, a_j)\}_{j=1}^{M}$ via maximum likelihood for $D$ epochs
    \STATE Update the priority of $s_j$ according to Equation~\ref{eq:priority} for all $j$
    \FOR{$G$ iterations}
    \STATE Sample segment index $k$ uniformly; Sample batch size $B$ from segment $k$ uniformly
    \STATE Update Q network as ${\phi_Q\leftarrow\phi_Q-\lambda_{\pi}\frac{1}{B}\sum_{i=1}^{B}w_i\cdot\nabla_{\phi_Q}J_{\pi}(\phi_Q, i)}$
    \STATE Update policy using ${J_{\pi}(\phi)=\frac{1}{B}\sum_{i=1}^{B}[D_{KL}(\pi||\exp\{Q^{\pi}-V^{\pi}\})]}$
    \ENDFOR
    \ENDFOR
  \end{algorithmic}
\end{algorithm}

\begin{figure}[!t]
  \centering
  \includegraphics[width=0.8\linewidth]{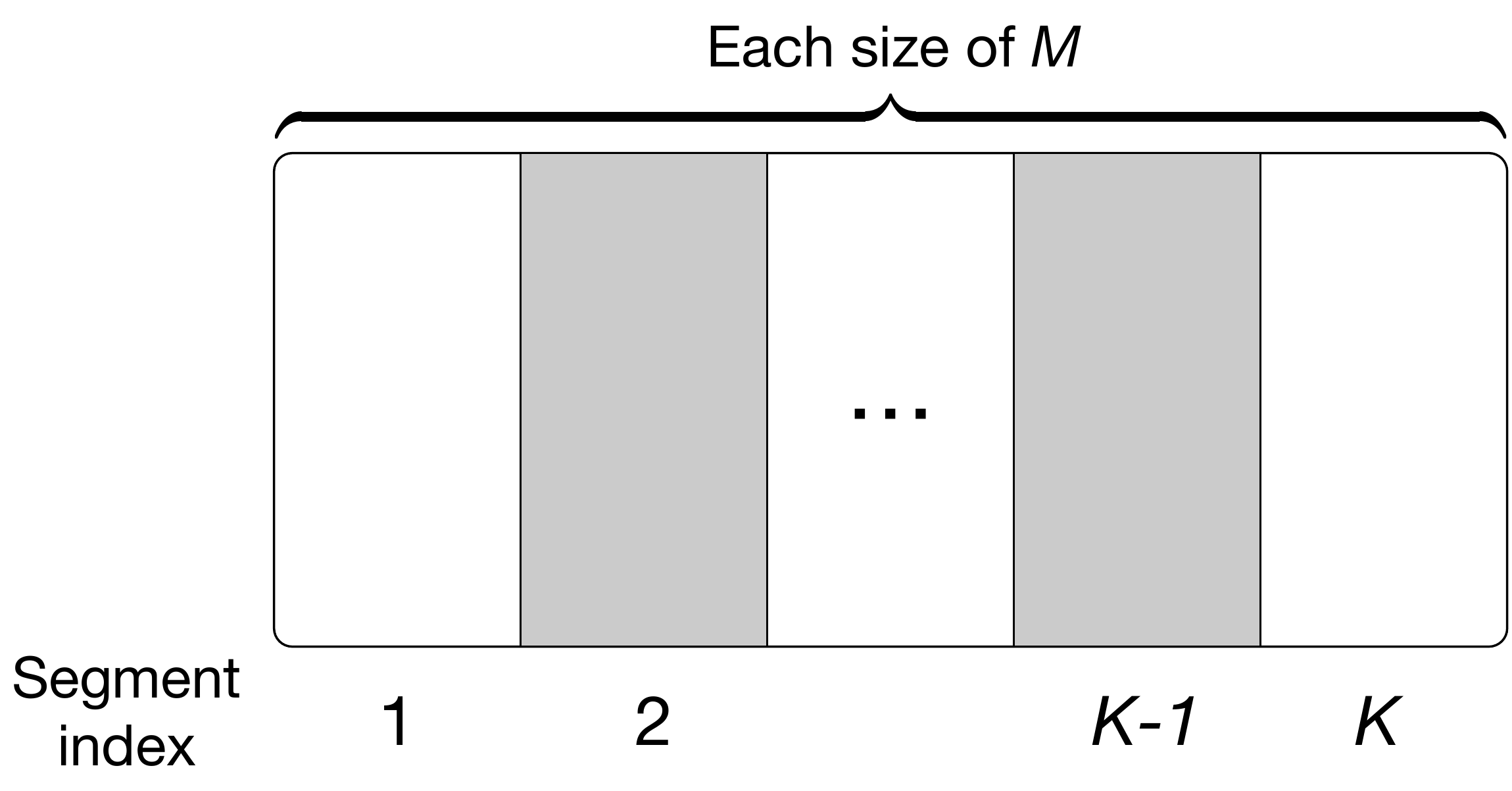}
  \caption{Segmented replay buffer for model generated rollouts. Each segment contains data sampled from the same environment state distribution. 
  }
  \label{fig:model_replay_buffer}
\end{figure}

\begin{figure*}[!t]
  \centering
  \begin{subfigure}{\linewidth}
    \centering
    \includegraphics[width=\linewidth]{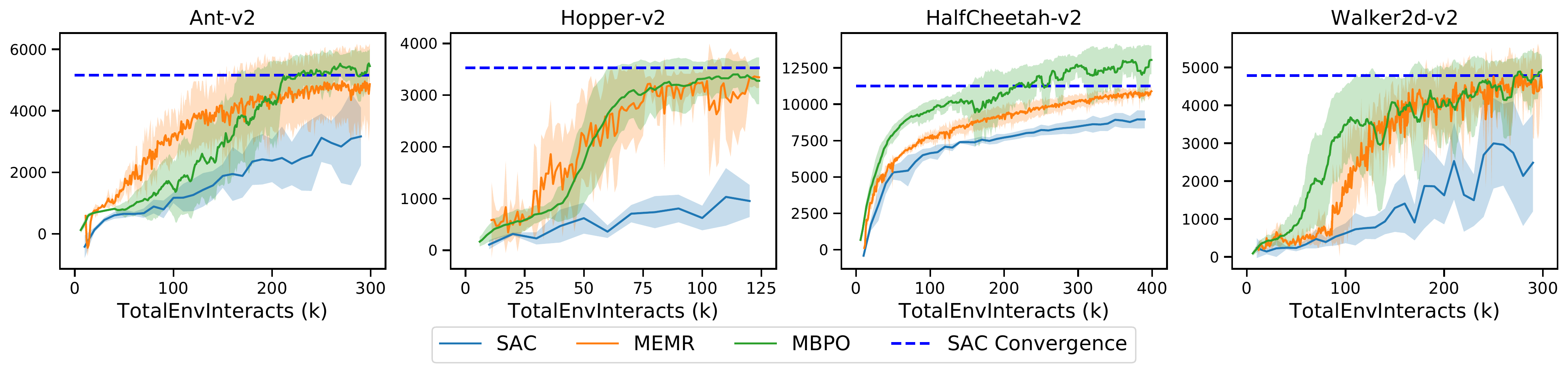}  
  \end{subfigure}
  \begin{subfigure}{\linewidth}
    \centering
    \includegraphics[width=\linewidth]{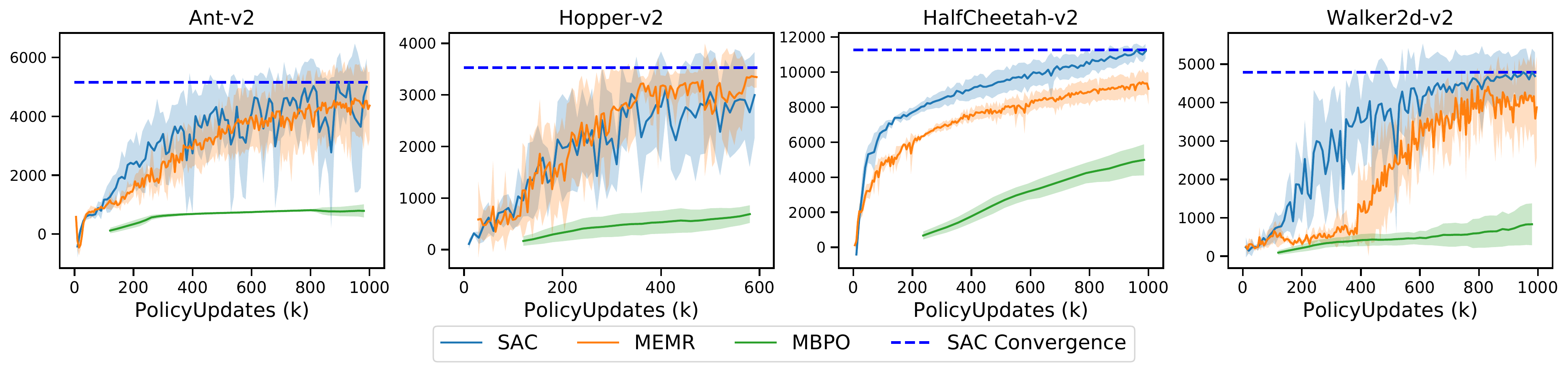}  
  \end{subfigure}
  \begin{subfigure}{0.5\linewidth}
    \centering
    \includegraphics[width=\linewidth]{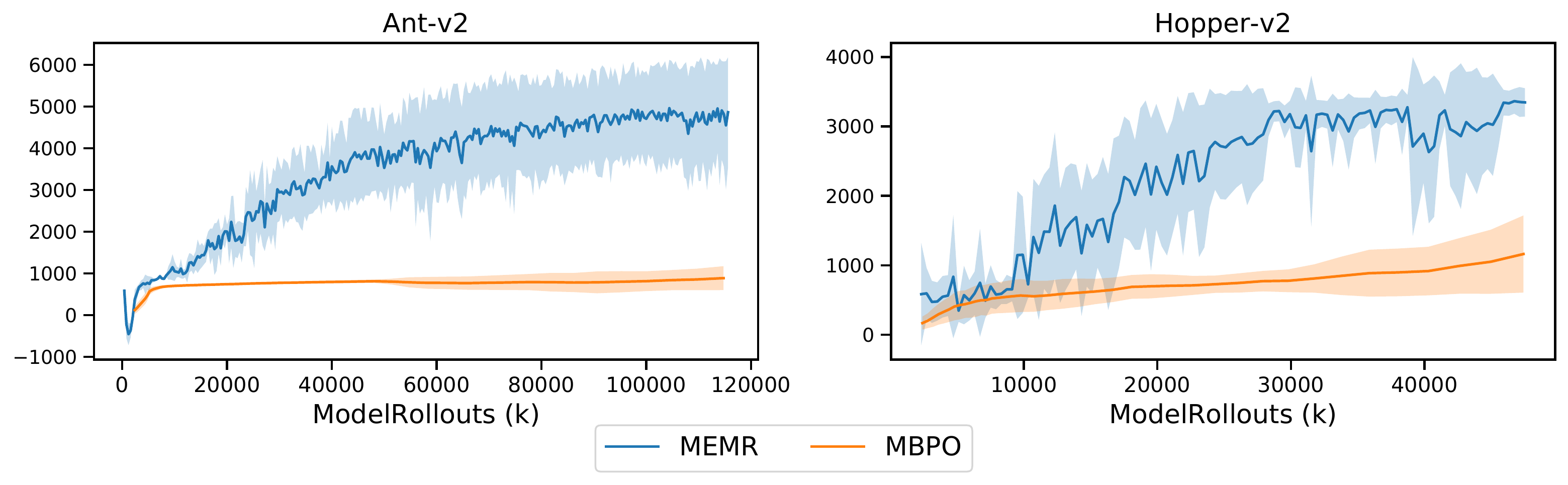}  
  \end{subfigure}
  \caption{Training curves of MEMR and two baselines. Solid curves depict the mean of five trials and shaded regions correspond to standard deviation among trials. The first row depicts the performance vs. the total number of environment interactions. We observe that MEMR matches the performance of state-of-the-art model-based and model-free algorithms. The second row shows the performance vs. the number of policy updates and we observe that MEMR converges as fast as SAC in terms of the number of updates. The third row shows that MEMR generates only a fraction of model rollouts compared to MBPO, which indicates far less training time.}
  \label{fig:results}
\end{figure*}

\paragraph{Stochastic prioritization.} Inspired by \cite{prioritized_experience_replay}, we only update the priority of the states that are just sampled to avoid an expensive full sweep before each sampling. An immediate consequence of this approach is that certain states with low priorities will not be sampled for a very long time. This potentially leads to overfitting. Following \cite{prioritized_experience_replay}, we use stochastic prioritization that interpolates between pure greedy and uniform sampling with the following probability of sampling state $i$:
\begin{align}
P(i)=\frac{p_{i}^{\alpha}}{\sum_{k}{p_{k}^{\alpha}}}
\label{eq:stochastic_sampling}
\end{align}
where $p_i\geq 0$ is the priority of state and action $i$. The exponent $\alpha$ determines how much prioritization is used, with $\alpha=0$ corresponding to the uniform case. According to Theorem~\ref{the:optimal_sampling_main}, we compute $p_i$ as 
\begin{align}
p_i=-\log(\sqrt{2\pi}\pi_{\psi}(a_i|s_i) \sigma(\pi_{\psi}(\cdot|s_i)))
\label{eq:priority}
\end{align}

\paragraph{Correcting the bias.} Using prioritized sampling introduces bias when fitting the Q network of the SAC. Inspired by \cite{prioritized_experience_replay}, we apply weighted importance-sampling (IS) when calculating the loss of the Q network, where the weight for sample $i$ is
\begin{align}
w_i={(\frac{1}{N}\cdot\frac{1}{P(i)})}^{\beta}
\end{align}

\paragraph{Segmented replay buffer.} According to Algorithm~\ref{alg:memr}, we update the priority after sampling states from the environment dataset to perform model rollouts. Thus, the sampling distribution of every $M$ model rollout generation is different. This leads to incorrect importance weights if we randomly sample a batch from the model dataset that contains data generated from different distributions to perform policy updates. To fix it, we introduce segmented replay buffer that group every $M$ rollouts in the same segment. During sampling for policy updates, we randomly sample a segment index, then sample a batch from that segment.

\paragraph{Training model derived policy distribution.} Fitting $\pi_{\psi}$ using $\dmodel$ via maximum likelihood to converge is costly since the size of $\dmodel$ is large and this operation must be performed every time we generate model rollouts. Since the data in model buffer is swapped rapidly, we treat it as an online learning procedure and only perform several gradient updates on the newly stored data.

\section{Experiments}
Our experimental evaluation aims to study the following questions: How well does MEMR perform on RL benchmarks, compared to state-of-the-art model-based and model-free algorithms in terms of sample efficiency, asymptotic performance and computation efficiency?

We evaluate MEMR on Mujoco benchmarks \cite{mujoco}. We compare our method with the state-of-the-art model-based method, MBPO \cite{mbpo}. As shown in Figure~\ref{fig:results}, MEMR matches the asymptotic performance of MBPO whereas MEMR only uses $\nicefrac{1}{4}$ policy updates and a fraction of model rollouts. It indicates that MEMR is more efficient in terms of model rollouts data used for policy updates. It also indicates orders of training speedup. Compared with the state-of-the-art model-free method, SAC \cite{sac}, MEMR matches the asymptotic performance and the data efficiency.


%% file: content/appendix.tex
\newpage
\onecolumn
\appendix
\appendixpage

\section{Maximum Entropy Sampling Criteria}
\label{sec:optimal_sampling_proof}
In this section, we present the proof of theorem~\ref{the:optimal_sampling_main}. 
We begin with a useful lemma as follows:
\begin{lemma}[Entropy Gain of Gaussian distribution]
  Suppose random variable $X\sim\mathcal{N}(\mu, \sigma^2)$, where $\mu$ and $\sigma$ are unknown. Now suppose we have observations $x_1,x_2,\cdots,x_N, N\gg 1$ and obtain an estimation of the distribution denoted as $\mathcal{N}_1(\mu_1,\sigma_{1}^2)$. If we have one more observation $t$ (variable) and obtain a new estimation of the distribution using $x_1,x_2,\cdots,x_N, t$, which is denoted as $\mathcal{N}_2(\mu_2,\sigma_{2}^2)$. Let the density of $\mathcal{N}_1$ be $f_1(x)$. Let the differential entropy of $\mathcal{N}_1$, $\mathcal{N}_2$ be $h_1$ and $h_2$. Let $g(t)=h_2-h_1$ Then, $g(t)=-{\log(\sqrt{2\pi} f_1(t)\sigma_{1})}/{N}$.
  \label{lemma:optimal_gaussian}
\end{lemma}
\begin{proof}
  According to the maximum likelihood estimation of Gaussian distribution, we obtain $\mu_1=(x_1+x_2+\cdots+x_N)/N$, $\sigma_{1}^2=\sum_{i=1}^{N}(x_i-\mu_1)^2/N$, $\mu_2(x)=(x_1+x_2+\cdots+x_N+t)/(N+1)$,  $\sigma_{2}^2(t)=(\sum_{i=1}^{N}(x_i-\mu_2)^2+(t-\mu_2)^2)/(N+1)$.
  Then,
  \begin{align}
  g(t)=\frac{1}{2}\log(\frac{\sigma_{2}(t)}{\sigma_{1}})^2&=\frac{1}{2}\log\frac{\sum_{n=1}^{N}(x_i-\mu_2)^2+(t-\mu_2)^2}{\sum_{n=1}^{N}(x_i-\mu_1)^2}\cdot \frac{N}{N+1}\\
  &=\frac{1}{2}\log(\frac{\sum_{n=1}^{N}(x_i-\mu_1+\mu_1-\mu_2)^2+(t-\mu_2+\mu_1-\mu_1)^2}{\sum_{n=1}^{N}(x_i-\mu_1)^2}\cdot \frac{N}{N+1})\\
  &=\frac{1}{2}\log((\frac{\sum_{n=1}^{N}(x_i-\mu_1)^2 +\sum_{n=1}^{N}(\mu_1-\mu_2)^2+2\sum_{n=1}^{N}(x_i-\mu_1)(\mu_1-\mu_2)}{\sum_{n=1}^{N}(x_i-\mu_1)^2}\\
  &+ \frac{(t-\mu_1)^2+(\mu_1-\mu_2)^2+2(t-\mu_1)(\mu_1-\mu_2)}{\sum_{n=1}^{N}(x_i-\mu_1)^2})\cdot \frac{N}{N+1})\nonumber\\
  &=\frac{1}{2}\log(\frac{N}{N+1}\cdot\frac{N\sigma_{1}^2+(N+1)(\mu_1-\mu_2)^2+(t-\mu_1)^2+2(t-\mu_1)(\mu_1-\mu_2)}{N\sigma_{1}^2})\\
  &=\frac{1}{2}\log(\frac{N}{N+1}(1+\frac{(t-\mu_1)^2}{N(N+1)\sigma_{1}^2}+\frac{(t-\mu_1)^2}{N\sigma_{1}^2}+\frac{2(t-\mu_1)(\mu_1-t)}{N(N+1)\sigma_{1}^2}))\\
  &=\frac{1}{2}\log(\frac{N}{N+1}(1+\frac{(t-\mu_1)^2}{(N+1)\sigma_{1}^2}))\\
  &\approx \frac{1}{2}\log(1+\frac{(t-\mu_1)^2}{N\sigma_{1}^2})\approx \frac{(t-\mu_1)^2}{2N\sigma_{1}^2}\\
  &=-\frac{\log(\sqrt{2\pi} f_1(t)\sigma_{1})}{N}
  \end{align}
\end{proof}
  
\begin{theorem}[Maximum Entropy Sampling Criteria]
  Assume $N_{\text{model}}\gg 1$ such that the state distribution of $\dmodel$ and $\mathcal{D}'_{\text{model}}$ are identical, then
  \begin{align}
  k=\argmin_i \log(\sqrt{2\pi}\pi_{\psi}(a_i|s_i) \sigma(\pi_{\psi}(\cdot|s_i)))
  \end{align}
  where $\pi_{\psi}(a_i|s_i)$ is the probability of model data policy at $(s_i,a_i)$, $\sigma(\pi_{\psi}(\cdot|s_i))$ is the standard deviation of the conditional distribution at $s_i$.
  \label{the:optimal_sampling_appendix}
\end{theorem}
\begin{proof}
  We begin by expanding the state and action joint entropy of the model dataset
  \begin{align}
  H(S',A')-H(S,A)&=-\int_{s'}p(s')\log p(s')ds'+\int_{s'}p(s')H(A'|S=s')ds'\\
  & +\int_{s}p(s)\log p(s)ds-\int_{s}p(s)H(A|S=s)ds\nonumber\\
  &\approx p(s_i)(H(A'|s_i)-H(A|s_i))\quad (\text{Since}\ N_{model}\gg 1,\ p(s')\approx p(s))
  \end{align}
  According to Lemma~\ref{lemma:optimal_gaussian}, we obtain
  \begin{align}
  H(S',A')-H(S,A)=p(s_i)(-\frac{\log(\sqrt{2\pi}\pi_{\psi}(a_i|s_i) \sigma(\pi_{\psi}(\cdot|s_i))}{C\cdot p(s_i)})=-\frac{\log(\sqrt{2\pi}\pi_{\psi}(a_i|s_i) \sigma(\pi_{\psi}(\cdot|s_i))}{C}
  \end{align}
  Note that $N$ in Lemma~\ref{lemma:optimal_gaussian} denotes the number of state $s_i$ in $\dmodel$, which is equal to $p(s_i)\cdot C$, where $C$ is the model dataset size. This is a rough density estimation and more accurate methods are left for future work. Thus,
  \begin{align}
  k=\argmax_i (H(S',A')-H(S,A))=\argmin_i \log(\sqrt{2\pi}\pi_{\psi}(a_i|s_i) \sigma(\pi_{\psi}(\cdot|s_i)))
  \end{align}
\end{proof}

\section{Related Work}
Model-free reinforcement learning (MFRL) learns optimal policy $\pi$ by directly taking gradient of the objective function \cite{npg,trpo,ppo} or estimate the state-action value function and derive the optimal policy \cite{dqn,sac}. These approaches require large amounts of environment interactions, which is not suitable for environments where sampling is expensive. On the contrary, model-based reinforcement learning (MBRL) learns a dynamics model and directly perform model predictive control (MPC) \cite{mb_mf,pets,pilco} or derives the policy using model generated rollouts \cite{me_trpo,mbpo,steve,slbo}.

Learning accurate dynamics model is often the bottleneck for MBRL to match the asymptotic performance of MFRL counterparts. Although Gaussian process is shown effective in low-dimensional data regime \cite{pilco,guided_policy_search}, it is hard to generalize to high-dimensional environments like \cite{mujoco, dqn}. \cite{mb_mf} first utilizes deterministic neural network dynamics model for model predictive control in robotics and \cite{pets} improves the idea with probabilistic ensemble models that matches the asymptotic performance with MFRL baselines. \cite{poplin} combines policy networks with online planning and achieves even superior performance on challenging benchmarks. Common MPC or planning methods include shooting method \cite{random_sampling_shooting}, cross-entropy method \cite{cross_entropy_method} and model predictive path integral control \cite{mppi}. Such planning methods would potentially over exploit the learned dynamics on long horizon predictions that may impair the performance. It is also computational expensive to perform in real time. In such cases, learning a policy, e.g. parameterized by a neural network, is desired for better generalization. 

Dyna-style MBRL utilizes learned dynamics to generate model rollouts to learn a good policy \cite{dyna}. \cite{model_based_value_expansion} and \cite{steve} utilizes the model to better estimate the value function in order to improve the sample efficiency. \cite{me_trpo} optimizes the policy network via policy gradient algorithm on the trajectories generated by the models. \cite{mbrl_atari} proposes video prediction network for model-based Atari games. \cite{slbo} develops a theoretical framework that provides monotonic improvement of the to a local maximum of the expected reward for MBRL. Model-based policy optimization (MBPO) \cite{mbpo} achieves state-of-the-art sample efficiency and matches the asymptotic performance of MFRL approaches. MBPO optimizes a policy with soft actor-critic (SAC) \cite{sac} under the data distribution collected by unrolling the learned dynamics model using the current policy. Our approach combines \cite{mbpo} and \cite{dyna} by proposing an non-trivial sampling approach to significantly reduce the number of policy updates and model rollouts that obtain asymptotic performance.
  
\section{Ablation Study}\label{appendix:ablation}
In this section, we make ablation studies to our proposed method. We primarily analyze how the performance of our algorithm changes by varying the size of the model dataset, the number of policy updates per environment step and the prioritization strength $\alpha$.

\begin{figure}
  \centering
  \includegraphics[width=\linewidth]{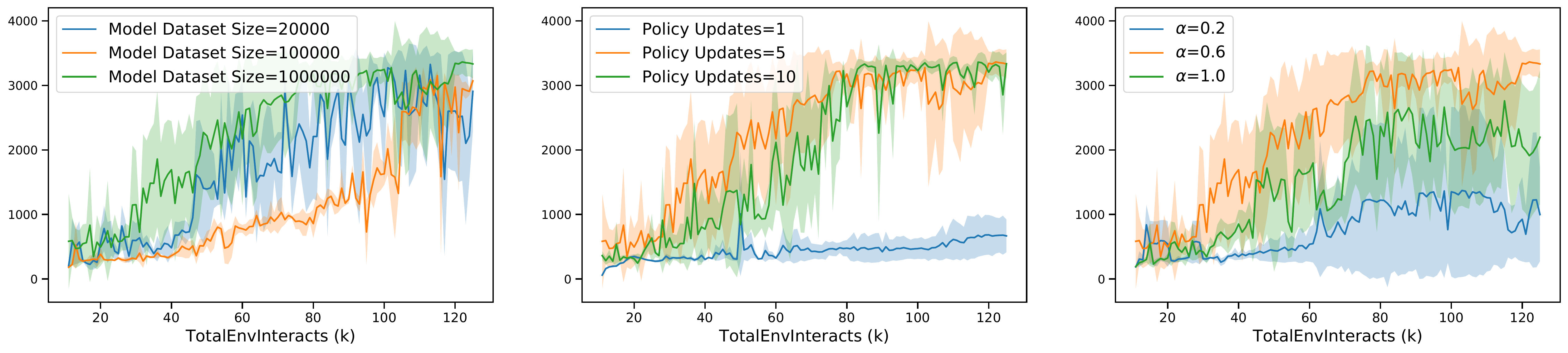}
  \caption{The results of ablation study. \textbf{Model Datset Size} refers to the size of $\dmodel$ shown in Algorithm~\ref{alg:memr}. \textbf{Policy Updates} refers to the number of policy updates per environment step. It indicates how informative the model rollouts are to the SAC agent. \textbf{Prioritization strength $\alpha$} is the exponent term used to calculate probability of states being sampled. The less it is, the more uniform the distribution would be.}
  \label{fig:ablation}
\end{figure}

\paragraph{Model dataset size}
The size of the model dataset $\dmodel$ affects how fast the algorithm converges. Since SAC is an off-policy algorithm, the same experience is expected to be revisited several times on average \cite{prioritized_experience_replay}. A small dataset would hinders the learning progress as the same transition only resides in the buffer for only a short period. On the other hand, a large model dataset would actually decrease the sample diversity of each batch used to perform policy updates.

\paragraph{Number of policy updates per environment step}
As shown in Figure~\ref{fig:results}, MEMR converges as fast as SAC in terms of policy updates. Surprisingly, we found that increasing the number of policy updates per environment step doesn't help to increase the convergence speed as shown in Figure~\ref{fig:ablation}. It indicates that much computation power is wasted in \cite{mbpo} on less informative model rollouts that barely help the learning of the value functions in SAC.

\paragraph{Prioritization strength}
Strong prioritization leads to overfit to local optimum whereas weak prioritization leads to less model rollouts diversity. As shown in Figure~\ref{fig:ablation}, we found that $\alpha=0.6$ works best for all benchmarks.

%
  
\section{Hyperparameter Settings}

\begin{table}[H]
  \caption{Hyperparameter setting for MEMR}
  \label{table:hyperparameter}
  \centering
  \begin{tabular}{ccccc}
    \toprule
    & HalfCheetah-v2 & Walker2d-v2 & Ant-v2 & Hopper-v2 \\
    \midrule
    Total number of steps & 400000 & 300000 & 300000 & 125000 \\
    Mode update frequency & \multicolumn{4}{c}{250}  \\
    Model rollouts per environment step ($M$) & \multicolumn{4}{c}{400} \\
    Prioritization strength ($\alpha$)  &  \multicolumn{4}{c}{0.6} \\
    Importance weights annealing ($\beta$)  & \multicolumn{4}{c}{Linear anneal from 0.4 to 1.0} \\
    Number of epochs to update model data policy $D$ & \multicolumn{4}{c}{2}\\
    Policy updates per environment step $G$ & \multicolumn{4}{c}{5} \\
    Size of model dataset & \multicolumn{2}{c}{3e6} & 6e6 & 1e6 \\
    \bottomrule
  \end{tabular}
\end{table}